\documentclass{article}

\PassOptionsToPackage{sort, numbers, compress}{natbib}
\usepackage[final]{nips_2018}

\pdfoutput=1
\usepackage[letterpaper]{geometry}
\usepackage[parfill]{parskip}
\usepackage[colorlinks]{hyperref}

\hypersetup{colorlinks,
            urlcolor=magenta,
            linkcolor=magenta,
            citecolor=magenta}

\usepackage{wrapfig}

\usepackage{graphicx}
\usepackage[utf8]{inputenc} 
\usepackage[T1]{fontenc}    
\usepackage{booktabs}       
\usepackage{amsfonts}       
\usepackage{nicefrac}       
\usepackage{microtype}      
\usepackage[parfill]{parskip}
\usepackage{algorithm,algorithmic}
\usepackage{amsmath,amsthm,amssymb,bbm}
\usepackage{mathtools}
\usepackage{cases}
\usepackage{comment}
\usepackage{subcaption}
\usepackage{appendix}
\usepackage{xspace}
\usepackage{enumitem}
\usepackage[english]{babel}

\usepackage{csquotes}

\newcommand{\vct}{\mathbf }

\usepackage{enumitem}
\setlist{nosep,after=\vspace{0.0\baselineskip},leftmargin=12pt}

\newtheorem{theorem}{Theorem}
\newtheorem{lemma}[theorem]{Lemma}

\newtheorem{prop}[theorem]{Proposition}

\newtheorem{cor}[theorem]{Corollary}

\title{Does mitigating ML's impact disparity \\
require treatment disparity?
}

\author{Zachary C. Lipton$^1$, Alexandra Chouldechova$^1$,
Julian McAuley$^2$\\
$^1$Carnegie Mellon University\\
$^2$University of California, San Diego\\
 \href{mailto:zlipton@cmu.edu}{\nolinkurl{zlipton@cmu.edu}},
   \href{mailto:acould@cmu.edu}{\nolinkurl{achould@cmu.edu}},
    \href{mailto:jmcauley@cs.ucsd.edu}{\nolinkurl{jmcauley@cs.ucsd.edu}}
}

\renewcommand{\P}{\mathbb{P}}
\newcommand{\E}{\mathbb{E}}

\DeclareMathOperator*{\argmin}

\begin{document}
\maketitle

\begin{abstract}
Following precedent in employment discrimination law,
two notions of disparity are widely-discussed 
in papers on fairness and ML. 
Algorithms exhibit \emph{treatment disparity}
if they formally treat 
members of protected subgroups differently;
algorithms exhibit \emph{impact disparity}
when outcomes differ across subgroups (even unintentionally).
Naturally, we can achieve impact parity 
through purposeful treatment disparity.
One line of papers aims to reconcile the two parities
proposing \emph{disparate learning processes} (DLPs).
Here, the sensitive feature is used during training 
but a \emph{group-blind} classifier is produced.
In this paper, we show that:
(i) when sensitive and (nominally) nonsensitive features are correlated,
DLPs will indirectly implement treatment disparity,
undermining the policy desiderata 
they are designed to address;
(ii) when group membership is \emph{partly} revealed by other features, DLPs induce within-class discrimination;
and (iii) in general, DLPs provide suboptimal trade-offs between accuracy and impact parity.
Experimental results on several real-world datasets 
highlight the practical consequences 
of applying DLPs.

\end{abstract}

\section{Introduction}
\label{section:intro}
Effective decision-making 
requires choosing among options 
given the available information.
That much is unavoidable, 
unless we wish to make trivial decisions.
In selection processes, 
such as hiring, university admissions, and loan approval, 
the options are people;
the available features include 
(but are rarely limited to) 
direct evidence of qualifications;
and decisions impact lives. 

Laws in many countries restrict 
the ways in which certain decisions can be made.
For example,
Title VII of the US Civil Rights Act \cite{CivilRights1964},
forbids employment decisions that discriminate 
on the basis of certain \emph{protected characteristics}.
Interpretation of this law 
has led to two notions of discrimination: 
\emph{disparate treatment} and \emph{disparate impact}.
\textit{Disparate treatment} 
addresses intentional discrimination, 
including 
(i) decisions explicitly based on protected characteristics; 
and (ii) intentional discrimination via proxy variables
(e.g~literacy tests for voting eligibility).
%
\textit{Disparate impact}
addresses facially neutral practices
that might nevertheless
have an ``unjustified adverse impact 
on members of a protected class'' \cite{CivilRights1964}.  
One might hope that detecting unjustified impact were as simple as detecting unequal outcomes.
However, absent intentional discrimination,
unequal outcomes can emerge due to correlations 
between protected and unprotected characteristics.
Complicating matters, unequal outcomes 
may not always signal unlawful discrimination \cite{AsianNYT}.

Recently, owing to the increased use
of machine learning (ML) 
to assist in consequential decisions,
the topic of quantifying and mitigating
ML-based discrimination 
has attracted interest 
in both policy and ML.
However, while the existing legal doctrine 
offers qualitative ideas,
intervention in an ML-based system 
requires more concrete 
formalism.
%
%
%
%
Inspired by the relevant legal concepts, technical papers have proposed several criteria 
to quantify discrimination. 
One criterion requires that 
the fraction given a positive decision be equal
across different groups.
Another criterion states that 
a classifier should be blind to the protected characteristic. Within the technical literature, these criteria are commonly referred to as \emph{disparate impact} and \emph{disparate treatment}, respectively.

In this paper, we call these technical criteria \emph{impact parity} and \emph{treatment parity} 
to distinguish them from their legal antecedents.  
The distinction between 
technical and legal terminology 
is important to maintain. 
While impact and treatment parity 
are inspired by legal concepts, 
technical approaches that achieve these criteria 
may fail to satisfy 
the underlying legal and ethical desiderata. 


We demonstrate one such disconnect
through DLPs,
a class of algorithms designed 
to simultaneously satisfy
treatment- and impact-parity criteria \cite{pedreshi2008discrimination,kamishima2011fairness,zafar2017fairness}.   
DLPs operate according to the following principle:
\emph{The protected characteristic may be used during training, but is not available to the model at prediction time.}
In the earliest such approach
the protected characteristic is used
to winnow the set of acceptable rules 
from an expert system \cite{pedreshi2008discrimination}.
Others incorporate the protected characteristic 
as either a regularizer, 
a constraint, 
or to preprocess the training data   \citep{kamiran2009classifying,kamiran2010discrimination,zafar2017fairness}.

These approaches are grounded in the premise 
that DLPs are acceptable in cases 
where using a protected characteristic 
as a direct input to the model
would constitute \emph{disparate treatment}
and thus be impermissible.  Indeed, DLPs in some sense operationalize a form of prospective fair ``test design'' that is well aligned with the ruling in \textit{Ricci v. DeStefano} \cite{kim2017auditing}.  
In this paper we investigate the utility of DLPs as a technical solution
and present the following cautionary insights:
\begin{enumerate}
\item When protected characteristics 
are redundantly encoded in the other features, 
sufficiently powerful DLPs 
can (indirectly) implement any form of treatment disparity.
\item When protected characteristics are partially encoded
DLPs induce within-class discrimination 
based on irrelevant features,
and can harm some members of the protected group. 
\item DLPs provide a suboptimal trade-off 
between accuracy and impact parity. 
\item While disparate treatment is by definition illegal, 
the status of treatment disparity is debated \citep{kim2017data}.
\end{enumerate}

\begin{figure*}
\begin{center}
\includegraphics[width=.8\linewidth]{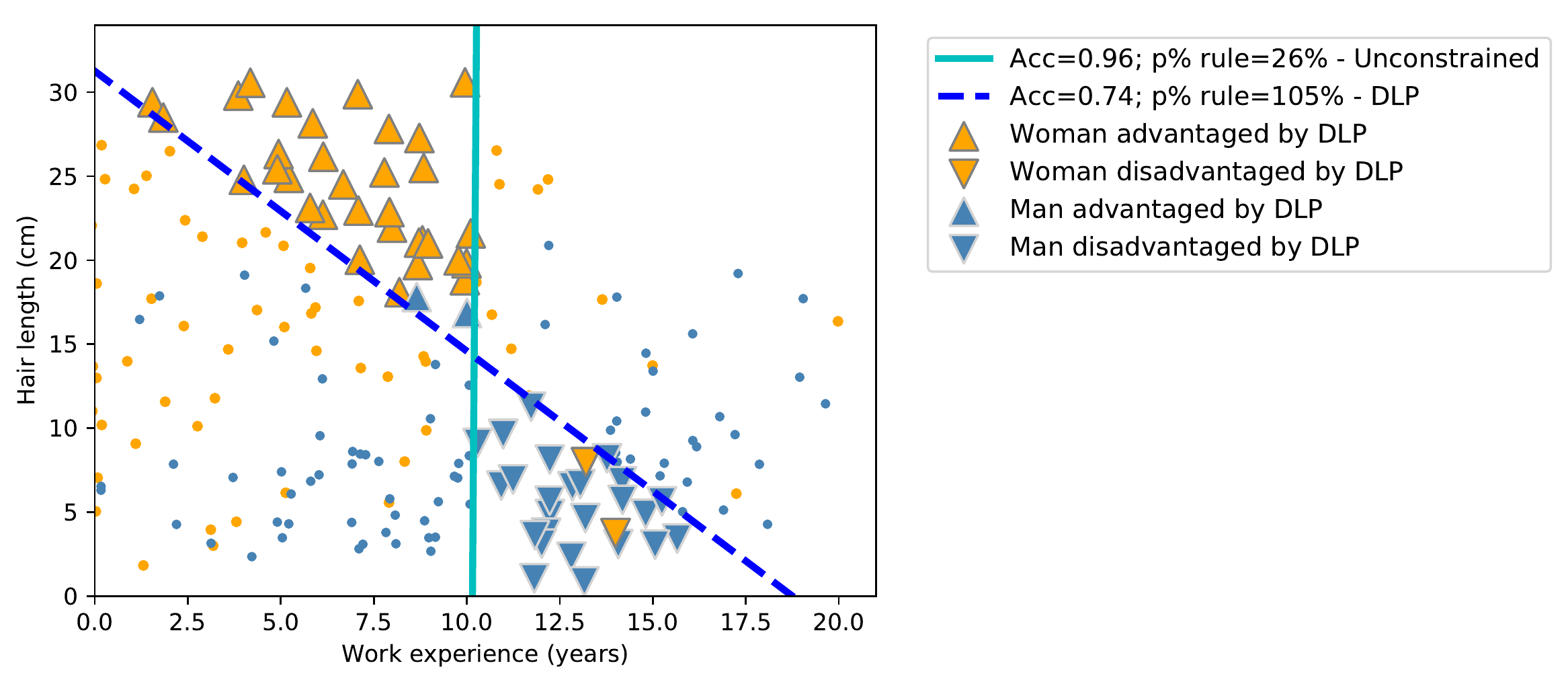}
\end{center}
\caption{Demonstration of a DLP's undesirable side effects on a simple example of hiring data (see \S \ref{sec:synthetic}).  
An unconstrained classifier (vertical line)
hires candidates based on work experience,
yielding higher hiring rates for men than for women.  
A DLP (dashed diagonal) achieves near-parity 
by differentiating based on an irrelevant attribute (hair length).  
The DLP \emph{hurts} some short-haired women,
flipping their decisions to reject, 
and helps some long-haired men.
}
\label{fig:intragroup-disparity}
\end{figure*}



\section{Disparate Learning Processes}
\label{section:fairness-algos}
To begin our formal description of the prior work,
we'll introduce some  notation.  
A dataset consists of $n$ \emph{examples}, 
or \emph{data points}
$\{\vct x_i \in \mathcal{X}, y_i \in \mathcal{Y} \}$,
each consisting of a feature vector $\vct x_i$ and a label $y_i$.  A supervised learning algorithm $f:\mathcal{X}^n \times \mathcal{Y}^n \rightarrow 
(\mathcal{X} \rightarrow [0,1])$ is a mapping from datasets to models.  The learning algorithm
produces a model 
$\hat{y}: \mathcal{X} \rightarrow \mathcal{Y}$,
which given a feature vector $\vct x_i$,
predicts the corresponding output $y_i$.
In this discussion
we focus on binary classification
($\mathcal{Y} = \{0,1\}$).

We consider probabilistic classifiers
which produce estimates $\hat{p}(\mathbf{x})$ of the conditional probability 
$\P(y=1 \mid \vct x)$
of the label given a feature vector $\mathbf{x}$. 
To make a prediction $\hat{y}(\vct x) \in \mathcal{Y}$
given an estimated probability $\hat{p}(\vct x)$ a threshold rule is used: $\hat{y}_i = 1 \text{ iff } \hat{p}_i > t$.
The optimal choice of the threshold $t$
depends on the performance metric being optimized.  In our theoretical analysis, we consider optimizing the \textit{immediate utility}
\citep{corbett2017algorithmic}, of which classification accuracy (expected $0-1$ loss) is a special case.   We will define this metric more precisely in the next section.

In formal descriptions of discrimination-aware ML,
a dataset  possesses a protected feature $z_i \in \mathcal{Z}$,
making each example a three-tuple $(\vct x_i, y_i, z_i)$.
The protected characteristic may be real-valued, like age, 
or categorical, like race or gender.  The goal of many methods in discrimination-aware ML is not only to maximize accuracy, but also to ensure some form of impact parity.
Following related work, 
we consider binary protected features 
that divide the set of examples 
into two groups $a$ and $b$.  Our analysis extends directly to settings with more than two groups.

Of the various measures of impact disparity,
the two that are the most relevant here are the Calders-Verwer gap and the p-\% rule. At a given threshold $t$, let
\mbox{$q_z = \frac{1}{n_z}\sum_{i : z_i=z}  \mathbbm{1}(\hat p_i > t)$}, where $n_z = \sum_{i}^{n} \mathbbm{1}(z_i = z)$. 
The \textbf{Calders-Verwer (CV) gap},
$q_a - q_b$,
is the difference between the proportions 
assigned to the positive class 
in the advantaged group $a$ and the disadvantaged  group $b$ \citep{kamishima2011fairness}.
The p-\% rule 
is a related metric\cite{zafar2017fairness}. 
Classifiers satisfy the \textbf{p-\% rule} 
if $q_b / q_a \ge p/100$.  

Many papers in discrimination-aware ML 
propose to optimize accuracy (or some other risk) 
subject to constraints on the resulting 
level of impact parity 
as assessed by some metric \citep{pedreshi2008discrimination,kamiran2010discrimination,dwork2017decoupled,bechavod2017learning,hardt2016equality,ritov2017conditional}.  
Use of DLPs presupposes
that using the protected feature $z$ 
as a model input is impermissible in this effort. 
Discarding protected features, however, 
does not guarantee impact parity \cite{dwork2012fairness}.  
DLPs incorporate $z$ in the learning algorithm, 
but without making it an input to the classifier.  
Formally, a DLP is a mapping: 
$\mathcal{X}^n \times \mathcal{Y}^n \times \mathcal{Z}^n \rightarrow (\mathcal{X} \rightarrow \mathcal{Y})$.
By definition, DLPs achieve treatment parity.  
However, satisfying \emph{treatment parity} in this fashion 
may still violate \emph{disparate treatment}.

\paragraph{Alternative approaches.} 
Researchers have proposed a number of other techniques for reconciling accuracy and impact parity. 
One approach consists of preprocessing
the training data 
to reduce the dependence between 
the resulting model predictions 
and the sensitive attribute \citep{kamiran2009classifying,kamiran2012data,
feldman2015certifying,adler2016auditing,
johndrow2017algorithm}. 
These methods differ in terms of 
which variables they affect
and the degree of independence achieved.  
\cite{kamiran2009classifying} 
proposed flipping negative labels 
of training examples form the protected group.
\cite{zemel2013learning} proposed 
learning representations (cluster assignments)
so that group membership 
cannot be inferred from cluster membership. 
\cite{feldman2015certifying} 
and \cite{johndrow2017algorithm} 
also construct representations
designed to be marginally independent from $Z$.  

\section{Theoretical Analysis}
\label{section:theory}
We present a set of simple theoretical results 
that demonstrate the optimality of treatment disparity, 
and highlight some properties of DLPs.  
We summarize our results as follows:
\begin{enumerate}
\item Direct treatment disparity on the basis of $z$ 
is the optimal strategy 
for maximizing classification accuracy\footnote{Our results are all presented in terms of a more general performance metric, of which classification accuracy is a special case.}
subject to CV and $p$-\% constraints.  
\item When $X$ fully encodes $Z$, a sufficiently powerful DLP is equivalent to treatment disparity.
\end{enumerate}
In Section \ref{section:empirical}, 
we empirically demonstrate a related point:
\begin{enumerate}[resume]
 \item When $X$ only partially encodes $Z$,
a DLP may be suboptimal and can induce intra-group disparity
on the basis of otherwise irrelevant features correlated with $Z$.
\end{enumerate}

\paragraph{Treatment disparity is optimal}
Absent impact parity constraints, the Bayes-optimal decision rule for minimizing expected $0-1$ loss (i.e., maximizing accuracy) is given by 
%
$d_{\mathrm{uncon}}^*(\vct x,z) = \delta( p_{Y|X, Z}(\vct x,z) \ge 0.5)$, where $\delta()$ is an indicator function.
%

We now show that the optimal decision rules 
in the CV and $p$-\% constrained problems 
have a similar form.  
The optimal decision rule 
will again be based on thresholding $p_{Y|X, Z}(\vct x,z)$, 
but at \textit{group-specific thresholds}. 
These rules can be thought of as operationalizing 
the following mechanism: 
Suppose that we start with the classifications 
of the unconstrained rule $d_{\mathrm{uncon}}^*(\vct x,z)$, 
and this results in a CV gap of $q_a - q_b  > \gamma$.  
To reduce the CV gap to $\gamma$ 
we have two mechanisms: We can
(i)  flip predictions from $0$ to $1$ in group $b$, and (ii) we can flip predictions from $1$ to $0$ in group $a$. 
The optimal strategy is to perform these flips on group $b$ cases that have the highest value of $p_{Y|X, Z}(\vct x,z)$ 
and group $a$ cases that have the lowest value of $p_{Y|X, Z}(\vct x,z)$.

The results in this section adapt the work of \cite{corbett2017algorithmic}, 
who establish optimal decision rules $d$ 
under 
exact parity.
In that work, the authors characterize 
the optimal decision rule $d = d(\vct x, z)$ 
that maximizes the \textit{immediate utility} 
$u(d, c) = \E[Yd(X,Z) - cd(X,Z)]$ for $(0 < c < 1)$,
under different exact parity criteria.  
We begin with a lemma showing 
that expected classification accuracy 
has the functional form 
of an immediate utility function.

\begin{lemma} \label{lemma:accuracy}
Optimizing classification accuracy is equivalent 
to optimizing immediate utility with $c = 0.5$. 
\end{lemma}

\begin{proof}
The expected accuracy of a binary decision rule $d(X)$ 
can be written as $\E[Yd(X) + (1-Y)(1-d(X))]$.  
Expanding and rearranging this expression gives
$$
\E[Yd(X) + (1-Y)(1-d(X))] = \E(2Yd(X) - d(X)) + \E(Y) + 1
	= 2u(d, 0.5) + \E(Y) + 1.
$$
The only term in this expression that depends on $d$ 
is the immediate utility $u$.  
Thus the decision rule that maximizes $u$ 
also maximizes accuracy.
\end{proof}

We note that the results in this section 
are related to the recent independent work 
of \cite{pmlr-v81-menon18a},
who derive Bayes-optimal decision rules 
under the same parity constraints we consider here, 
working instead with the \textit{cost-sensitive risk}, 
$
\mathrm{CS}(d;c) = \pi(1-c)\mathrm{FNR}(d) + (1-\pi)c\mathrm{FPR}(d),
$
where $\pi = \P(Y = 1)$.  
One can show that $u(d, c) = -\mathrm{CS}(d;c) + \pi(1-c)$, 
and hence the problem of maximizing immediate utility considered here 
is equivalent to minimizing cost-sensitive risk 
as in \cite{pmlr-v81-menon18a}.  
In our case, it will be more convenient 
to work with the immediate utility.  

For the next set of results,
we follow \cite{corbett2017algorithmic} 
and assume that $p_{Y\mid X,Z}(X,Z)$, 
viewed as a random variable, 
has positive density on $[0,1]$.  
This ensures that the optimal rules 
are unique and deterministic 
by disallowing point-masses of probability 
that would necessitate tie-breaking 
among observations with equal probability.  
The first result that we state 
is a direct corollary 
of two results in \cite{corbett2017algorithmic}.  
It considers the case 
where we desire exact parity, i.e., 
that $q_a = q_b$.  

\begin{cor}
 The optimal decision rules $d^*$ 
 under various parity constraints 
 have the following form 
 and are unique up to a set of probability zero: 
 \begin{enumerate}
 \item Among rules satisfying statistical parity (the 100\% rule), the optimum is
$d^*(\vct x,z) = \delta(p_{Y|X, Z}(\vct x,z) \ge t_z),$
 where $t_z \in [0,1]$ are constants 
 that depend only on group membership $z$. 
 \item Among rules that have equal false positive rates across groups, the optimum is
 $d^*(\vct x,z) = \delta(p_{Y|X, Z}(\vct x,z) \ge s_z),$
 where $s_z$ are constants 
 that depend only on group membership $z$ 
 (but are different from $t_z$).
 \item (1) and (2) continue to hold 
 even in the resource-constrained setting 
 where the overall proportion of cases classified as positive is constrained.
 \end{enumerate}
\end{cor}

\begin{proof}
(1) and (2) are direct corollaries of Lemma \ref{lemma:accuracy} combined with Thm 3.2 and Prop 3.3 of \cite{corbett2017algorithmic}. 
\end{proof}

The next set of results establishes optimality 
under general $p$-\% and CV rules.

\begin{prop}
	Under the same assumptions as above, 
    the optimum among rules 
    that satisfy the CV constraint $0 \le q_a - q_b < \gamma$ 
    or the $p$-\% rule also has the form
$d^*(\vct x,z) = \delta(p_{Y|X, Z}(\vct x,z) \ge t_{z}),$
 where $t_{z} \in [0,1]$ are constants 
 that depend on the group membership $z$, 
 and on the choice of constraint parameter $\gamma$ or $p$.  
 The thresholds $t_z$ are different 
 for the CV constraint and $p$-\% rule.
\label{prop:cv_opt}
\end{prop}

\begin{proof}
Suppose that the optimal solution under the CV or $p$-\% rule constraint classifies proportions $q_a$ and $q_b$
of the advantaged and disadvantaged groups, respectively,
to the positive class. 
As shown in \citet{corbett2017algorithmic}, 
we can rewrite the immediate utility as 
\[
 u(d,0.5) = \E[d(X, Z)(p_{Y|X, Z} - 0.5)].
\]
Thus the utility will be maximized 
 when $d^*(X,Z) = 1$ 
for the $q_z$ proportion of individuals in each group 
that have the highest values of $p_{Y|X, Z}$.  
Since the optimal values of $q_z$ 
may differ between the CV-constrained solution 
and the $p$-\% solution, 
the optimal thresholds may differ as well.
\end{proof}

Our final result
shows that a decision rule 
that does not directly use $z$ as an input variable 
or for determining thresholds 
will have lower accuracy 
than the optimal rule 
that uses this information.  
That is, we show that DLPs are suboptimal 
for trading off accuracy and impact parity.

\begin{theorem}
	Let $d^*(\vct x, z)$ be the optimal decision rule under a the CV-$\gamma$ or $p$-\% constraint. 
    Let $d_{\mathit{DLP}}(\vct x)$ be the optimal solution to a DLP. 
        If $d(\vct x,z)$ and $d_{\mathit{DLP}}(\vct x)$ satisfy CV or $p$-\% constraints 
        with the same $q_a$ and $q_b$, 
        the DLP solution results in lower or equal accuracy 
        (equal only if the solutions are the same.)
\end{theorem}

\begin{proof}
 From Proposition \ref{prop:cv_opt},
 we know that the unique accuracy-optimizing solution is given by 
 $d^*(\vct x,z) = \delta(p_{Y|X, Z}(\vct x,z) \ge t_{z}),$
 where $t_{z}$ is the 1 - $q_{z}$ quantile of $p_{Y|X, Z}$.  
 The difference in immediate utility 
 between the two decision rules 
 can be expressed as follows:
 \small
 \begin{align*}
 	&\E[d^*(X, Z)(p_{Y|X, Z} - .5)] - \E[d_{\mathit{DLP}}(X)(p_{Y|X, Z} - .5)] \\
     &= \E[(d^*(X, Z) - d_{\mathit{DLP}}(X))(p_{Y|X, Z} - 0.5)] \\
     &= \E[p_{Y|X, Z}\!-\! \!.5\!\! \mid\!\! d^*\! =\! 1, d_{\mathit{DLP}}\! =\! 0]\P(d^*\!\! =\! 1, d_{\mathit{DLP}}\! =\! 0)
     \!-\!\E[p_{Y|X, Z}\! -\!\! .5\! \!\mid\!\! d^*\!\! =\! 0, d_{\mathit{DLP}}\! =\! 1]\P(d^*\!\! =\! 0, d_{\mathit{DLP}}\! =\! 1) \\
     &= \bigl( \E[p_{Y|X, Z} - .5 \mid d^* = 1, d_{\mathit{DLP}} = 0]  - \E[p_{Y|X, Z} - .5 \mid d^* = 0, d_{\mathit{DLP}} = 1] \bigr) \P(d^* = 1, d_{\mathit{DLP}} = 0) \\
     &\ge 0
 \end{align*}
 \normalsize
 The final inequality follows 
 since
 $d^*(X, Z) = 1$ for the highest values of $p_{Y|X, Z}$, 
 so $p_{Y|X,Z}$ is stochastically greater 
 on the event $\{d^*\!=\! 1, d_{\mathit{DLP}}\! =\! 0\}$ 
 than on $\{d^*\! =\! 0, d_{\mathit{DLP}}\! =\! 1\}$.  
 Note that equality holds only if 
 $\P(d^*\! =\! 1, d_{\mathit{DLP}}\! =\! 0) = 0$, i.e., 
 if the two rules are (almost surely) equivalent.
\end{proof}

Our results
continue to hold 
under ``do no harm'' constraints, where we require that any individual in the disadvantaged group who was classified as positive under 
the unconstrained rule $d_\mathrm{uncons}(\vct x,z)$ remains positively classified.  This corresponds to the setting where 
the proportion of cases 
in the disadvantaged group classified as positive 
is constrained to be no lower 
than the proportion under the unconstrained rule 
(or no lower than some fixed value $q_a^\mathrm{min}$).  
Such constraints impose an upper bound on the optimal thresholds $t_b$, but do not change the structure of the optimal rules.

\paragraph{Functional equivalence when protected characteristic is redundantly encoded.}
Consider the case where the protected feature $z$ 
is redundantly encoded in the other features $\vct x$.  
More precisely, 
suppose that there exists a known subcomputation $g$ 
such that $z = g(\vct x)$.  
This allows for any function of the data $f(\vct x,z)$ 
to be represented as a function of $\vct x$ alone via $\tilde f(\vct x) = f(\vct x, g(\vct x))$.   
While it remains the case that $\tilde f(\vct x)$ does not directly use $z$ as an input variable---and thus satisfies treatment parity---$\tilde f$ should be no less legally suspect from a \textit{disparate treatment} perspective 
than the original function $f$ that uses $z$ directly. 
The main difference for the purpose of our discussion 
is that $\tilde f$, resulting from a DLP, may technically satisfy treatment parity, while $f$ does not.  

While this form of ``strict'' redundancy is unlikely, characterizing this edge case is
important for considering whether DLPs should have different legal standing vis-a-vis disparate treatment than methods
that use $z$ directly. This is particularly relevant if one thinks of the `practitioner' in question as having discriminatory
intent. Furthermore, the partial encoding of the protected attribute is commonplace in settings where discrimination
is a concern (as with gender in our experiment in \S \ref{section:empirical}). Indeed, the very premise of DLPs requires that $\vct x$ is
significantly correlated with $z$. Moreover, DLPs provide an incentive for practitioners to game the system by adding
features that are predictive of the protected attribute but not necessarily of the outcome, as these would improve the
 DLP's performance.

\paragraph{Within-class discrimination when protected characteristic is partially redundantly encoded.}
When the protected characteristic is partially encoded in the other features, disparate treatment may induce within-class discrimination by applying the benefit of the affirmative action unevenly, and can even harm some members of the protected class.
Next
we demonstrate this phenomenon empirically using 
(synthetically biased) university admissions data and several public datasets. 
The ease of producing such examples might convince the reader that the 
varied effects of intervention with a DLP 
on members of the disadvantaged group raises practice and policy concerns about DLPs.


\section{Empirical Analysis}
\label{section:empirical}

This preceding analysis demonstrates
several theoretical advantages 
to increasing impact parity via treatment disparity:
\begin{itemize}
\item \textbf{Optimality: } 
As demonstrated for CV score and for $p$-$\%$ rule, 
intervention via per-group thresholds 
maximizes accuracy subject to an impact parity constraint.
\item \textbf{Rational ordering:} Within each group, individuals with higher probability 
of belonging to the positive class
are always assigned to the positive class 
ahead of those with lower probabilities.
\item \textbf{Does no harm to the protected group:} The treatment disparity intervention 
can be constrained to only benefit members of the disadvantaged class.
\end{itemize}

DLPs attempt to produce a classifier 
that satisfies the parity constraints,
by relying upon the proxy features 
to satisfy the parity metric. 
Typically, this is accomplished 
either by introducing constraints 
to a convex optimization problem, 
or by adding a regularization term 
and tuning the corresponding hyper-parameter. 
Because the CV score and $p$-$\%$ rule 
are non-convex in model parameters 
(scores only change 
when a point crosses the decision boundary),
\cite{kamishima2011fairness, zafar2017fairness}
introduce convex surrogates 
aimed at reducing the correlation
between the sensitive feature and the prediction.

These approaches presume that the proxy variables 
contain information about the sensitive attribute.
Otherwise, the parity could only be satisfied 
via a trivial solution 
(e.g.~assign either \emph{everyone} or \emph{nobody} 
to the positive class). 
So we must consider two scenarios: 
(i) the proxy variables $\mathbf{x}$ 
fully encode $z$,
in which case, a sufficiently powerful DLP 
will implicitly reconstruct $z$,
because this gives the optimal solution 
to the impact-constrained objective;
and (ii) $\vct x$ doesn't fully capture $z$, or the DLP is unable to recover $z$ from $\vct x$,
in which case the DLP may be sub-optimal,
may violate rational ordering within groups, 
and may harm members of the disadvantaged group.

%
%
\subsection{Synthetic data example: work experience and hair length in hiring} \label{sec:synthetic}
To begin, we illustrate our arguments empirically 
with a simple synthetic data experiment.  
To construct the data, 
we sample $n_{\text{all}} = 2000$ total observations 
from the data-generating process described below.  
70\% of the observations are used for training, 
and the remaining 30\% are reserved for model testing.  
\begin{align*}
z_i &\sim \mathrm{Bernoulli}(0.5) \\
\mathrm{hair\_length}_i \mid z_i = 1 &\sim  35 \cdot \mathrm{Beta}(2,2) \\ 
\mathrm{hair\_length}_i \mid z_i = 0 &\sim  35 \cdot \mathrm{Beta}(2,7) \\ 
\mathrm{work\_exp}_i \mid z_i &\sim  \mathrm{Poisson}(25 + 6z_i) - \mathrm{Normal}(20, \sigma = 0.2) \\
y_i \mid \mathrm{work\_exp}  &\sim 2 \cdot \mathrm{Bernoulli}(p_i) - 1\text{,} \\ 
\text{where  } p_i &= 1 / \left(1 + \exp[-(-25.5 + 2.5\mathrm{work\_exp})]\right)
\end{align*}
This data-generating process 
has the following key properties: 
(i) the historical hiring process 
was based solely on the number of years of work experience; 
(ii) because women on average have fewer years of work experience than men (5 years vs. 11), 
men have been hired at a much higher rate than women; 
and (iii) women have longer hair than men, 
a fact that was irrelevant to historical hiring practice.

Figure \ref{fig:intragroup-disparity} 
shows the test set results of applying a DLP 
to the available historical data 
to equalize hiring rates between men and women.  
We apply the DLP proposed by \citet{zafar2017fairness}, 
using code available from the authors.\footnote{\url{https://github.com/mbilalzafar/fair-classification/}} 
While the DLP nearly equalizes hiring rates 
(satisfying a $105$-\% rule), 
it does so through a problematic 
within-class discrimination mechanism.  
The DLP rule advantages individuals 
with longer hair over those with shorter hair 
and considerably longer work experience.  
We find that several women 
who would have been hired under historical practices, 
owing to their 12+ years of work experience, 
would not be hired under the DLP 
due to their short hair 
(i.e., their male-like characteristics captured in $\vct x$). Similarly, several men,
who would not have been hired based on work experience alone, are advantaged by the DLP due to their longer hair  
(i.e., their `female-like' characteristics in $\vct x$).  
The DLP violates rational ordering, 
and harms some of the most qualified individuals 
in the protected group.  
Group parity is achieved at the cost of individual unfairness.

Granted, we might not expect
factors such as hair length 
to knowingly be used as inputs 
to a typical hiring algorithm.  
We construct this toy example to
illustrate a more general point:  
since DLPs do not have direct access 
to the protected feature, 
they must infer from the other features
which people are most likely to belong to each subgroup.  
Using the protected feature directly 
can yield more reasonable policies: 
For example, by applying per-group thresholds,
we could hire the highest rated individuals in each group, 
rather than distorting rankings within groups
based on how female/male individuals \emph{appear} 
to be from their other features.

\begin{figure*}[t]
\begin{center}
\includegraphics[width=0.33\linewidth]{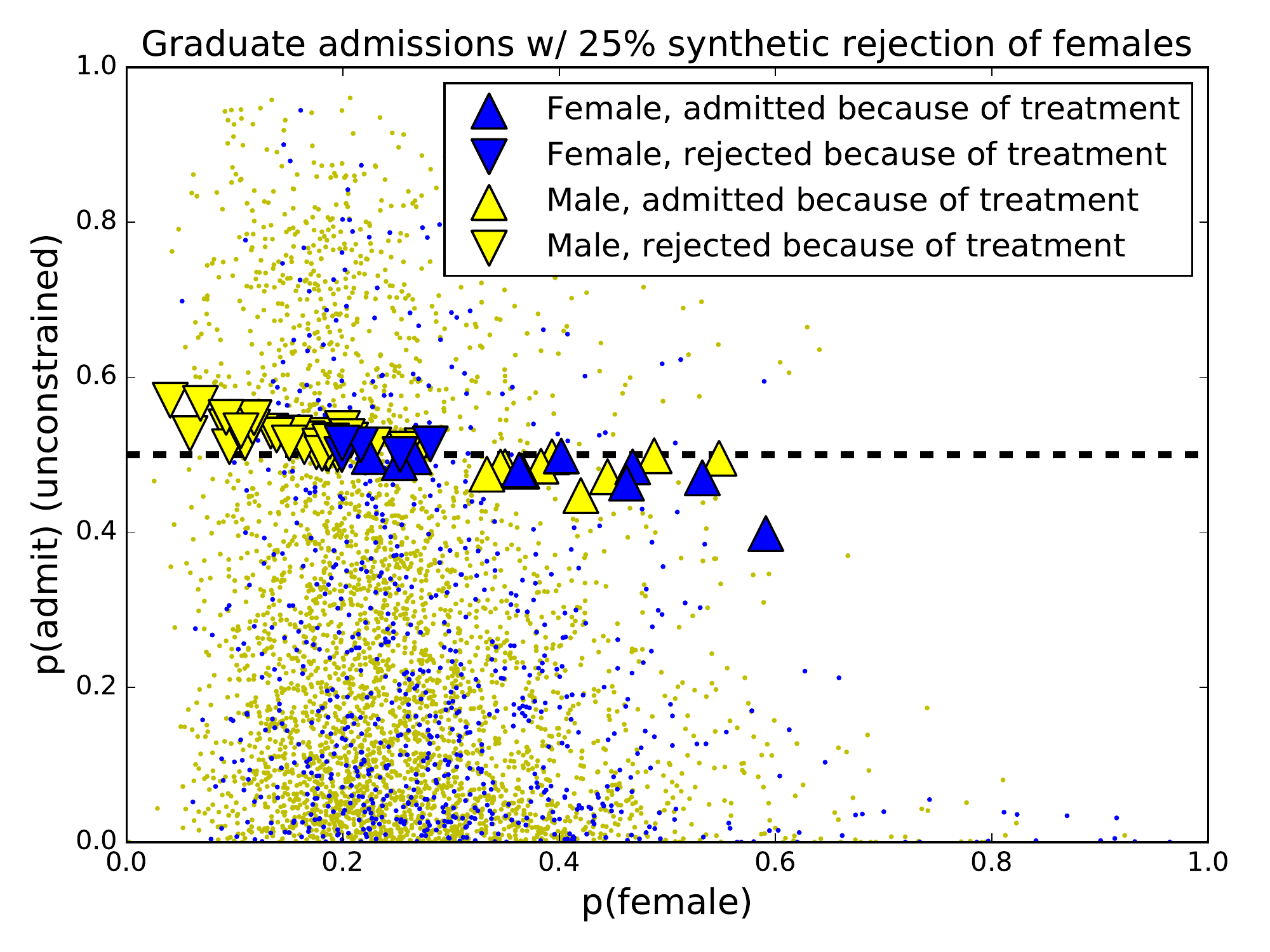}\includegraphics[width=0.33\linewidth]{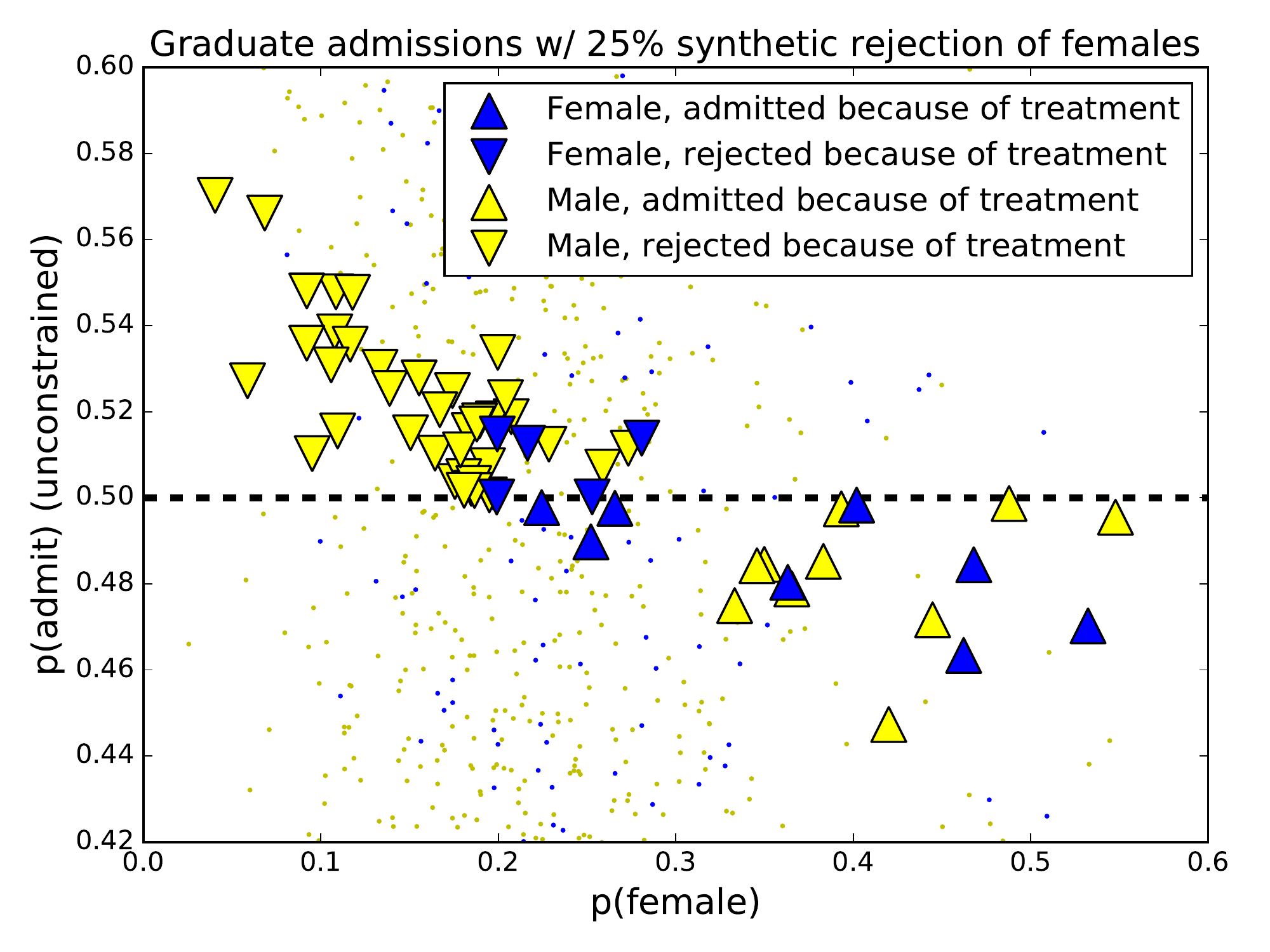}\includegraphics[width=0.33\linewidth]{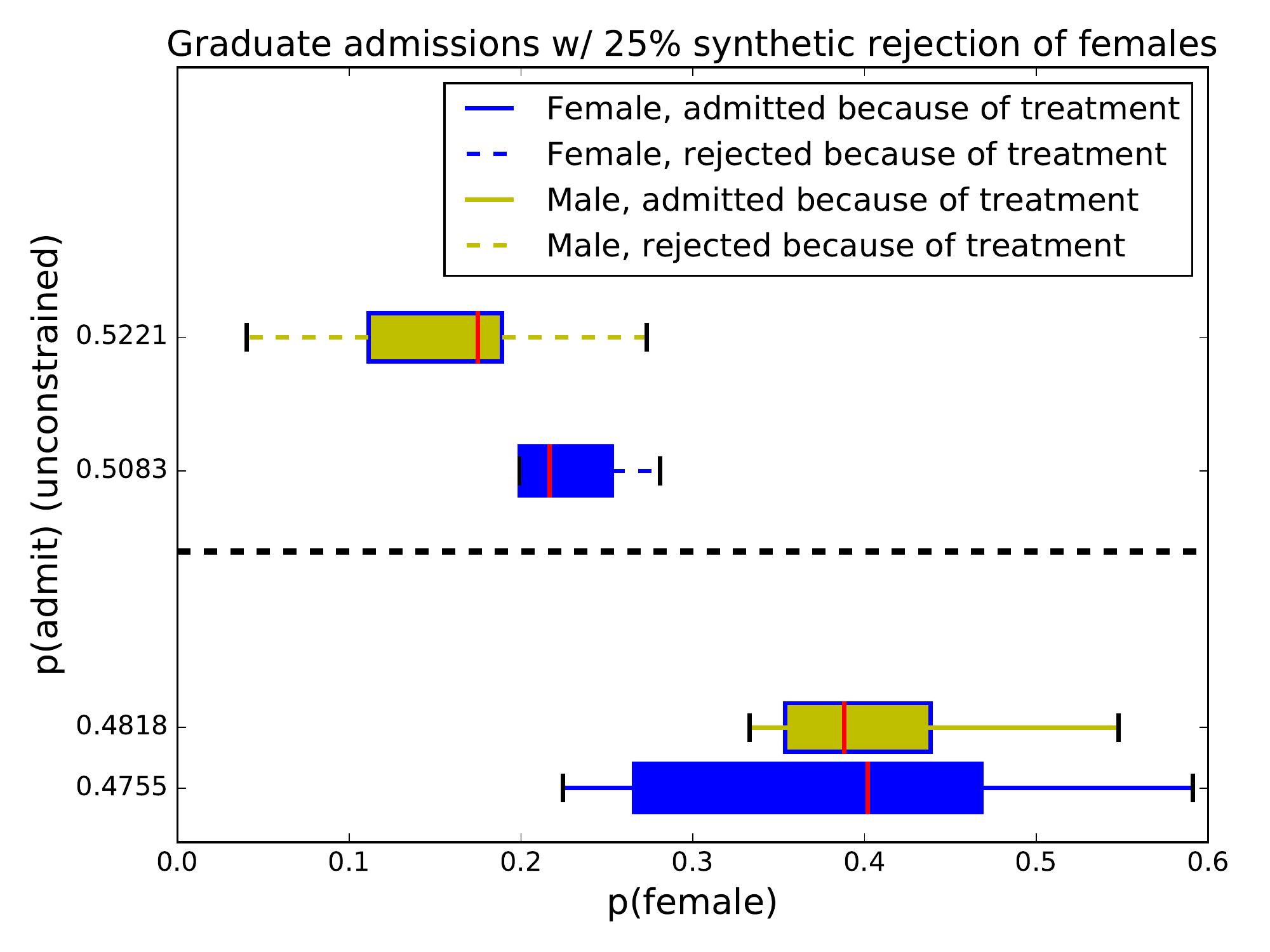}
\end{center}
\caption{
(left) probability of the sensitive variable 
versus (unconstrained) admission probability, 
on unseen test data. 
Downward triangles indicate individuals 
rejected
\emph{only} after applying the DLP (``treatment''), 
while upward triangles indicate individuals 
accepted \emph{only} by the DLP.
The remaining $\sim$4,000 blue/yellow dots
indicate people whose decisions are not altered. 
Many students benefiting from the DLP 
are males who `look like' females based on other features, 
whereas females who `look like' males are hurt by the DLP. Detail view (center) and summary statistics (right) of the same plot.}
\label{fig:admissions}
\end{figure*}


\subsection{Case study: Gender bias in CS graduate admissions}
\label{sec:admissions}
\label{sec:threshold}


For our next example,
we demonstrate a similar result 
but this time by analyzing real data 
with synthetic discrimination, to empirically demonstrate our arguments.
We consider a sample of $\sim$9,000 students 
considered for admission 
to the MS program 
of a large US university
over an 11-year period.
Half of the examples are withheld for testing. 
Available attributes include
basic information, 
such as country of origin, interest areas, 
and gender, 
as well as quantitative fields 
such as GRE scores. 
Our data also includes a label 
in the form of an `above-the-bar' decision 
provided by faculty reviewers.
%
Admission rates for male and female applicants 
were observed to be within 1\% of each other.
So, to demonstrate the effects of DLPs, 
we corrupt the data with \emph{synthetic discrimination}.
Of all women who were admitted, i.e., $z_i = b, y_i = 1$,
we flip 25\% of those labels to $0$:
giving noisy labels $\bar{y}_i = y_i \cdot \eta$, 
for $\eta \sim \textit{Bernoulli}(.25)$.
This simulates 
historical bias in the training data.

We then train three logistic regressors: 
(1) To predict the (synthetically corrupted) labels $\bar y_i$
from the non-sensitive features 
$x_i$; 
(2) The same model, applying the DLP of \citep{zafar2017fairness}; and 
(3) A model 
to predict the sensitive feature  $z_i$
from the non-sensitive features $x_i$. 
The data contains limited information 
that predicts gender, 
though such predictions can be made better 
than random (AUC=0.59) 
due to different rates of gender imbalance 
across (e.g.)~countries and interests.



Figure~\ref{fig:admissions} (left) shapes 
our basic intuition for what is happening:
Considering the probability of admission 
for the unconstrained classifier (y-axis), 
students whose decisions are `flipped' 
(after applying the fairness constraint)
tend to be those close to the decision boundary. 
Furthermore, students \emph{predicted} to be male (x-axis)
tend to be flipped to the negative class 
(left half of plot) 
while students \emph{predicted} to be female 
tend to be flipped to the positive class 
(right half of plot). 
This is shown in detail 
in Figure~\ref{fig:admissions} (center and right). 
Of the 43 students whose decisions are flipped to `non-admit,' 5 are female, 
each of whom has `male-like' characteristics 
according to their other features
as demonstrated in our synthetic hair-length example.
Demonstrated here with real-world data,
the DLP both disrupts the within-group ordering
and violates the \emph{do no harm} principle 
by disadvantaging some women who, but for the DLP, 
would have been admitted. 

\paragraph{Comparison with Treatment Disparity.}
To demonstrate the better performance 
of per-group thresholding, 
we implement a simple decision scheme
and compare its performance to the DLP.

Our thresholding rule
for maximizing accuracy 
subject to a $p$-$\%$ rule works as follows:
Recall that the $p$-$\%$ rule requires 
that $q_b/q_a > p/100$, which can be written as
$\frac{p}{100} q_a - q_b < 0$.
We denote the quantity $\frac{p}{100} q_a - q_b$ as the $p$-gap.
To maximize accuracy subject to satisfying the $p$-$\%$ rule,
we construct a score
that quantifies reduction in $p$-gap 
per reduction in accuracy.
Starting from the accuracy-maximizing classifications $\hat y$
(thresholding at $.5$),
we then flip those predictions 
which close the gap fastest:

\begin{enumerate}
\item Assign each example with $\lbrace \tilde{y}_i =0, z_i=b\rbrace$ or $\lbrace\tilde{y}_i = 1, z_i=a \rbrace$,
a score $c_i$ 
equal to the reduction in the p-gap 
divided by the reduction in accuracy:
\begin{enumerate}
\item For each example in group $a$ with initial $\hat{y}_i=1$,
$c_i = \frac{p}{100n_a(2\hat{p}_i-1)}$.
\item For each example in group $b$ with initial $\hat{y}_i=0$,
$c_i = \frac{1}{n_b(1-2\hat{p}_i)}$.\\[-0.5\baselineskip]
\end{enumerate}
\item Flip examples in descending order 
according to this score 
until the desired CV-score is reached.
\end{enumerate}
These scores do not change after each iteration, so the greedy policy leads to optimal flips (equivalently, optimal classification thresholds).

\begin{table}
\centering
\caption{Statistics of public datasets.\label{tab:datastats}}
\footnotesize 
\begin{tabular}{llllr}
\toprule
\textbf{dataset} & \textbf{source} & \textbf{protected feature} & \textbf{prediction target} & $n$ \\
\midrule
Income & UCI \citep{uci_income} & Gender (female) & income $>$ \$50k & 32,561\\
Marketing & UCI \citep{uci_marketing} & Status (married) & customer subscribes & 45,211\\
Credit & UCI \citep{uci_default} & Gender (female) & credit card default & 30,000\\
Employee Attr. & IBM \citep{ibm_data} & Status (married) & employee attrition & 1,470\\
Customer Attr. & IBM \citep{ibm_data} & Status (married) & customer attrition & 7,043\\
\bottomrule
\end{tabular}
\normalsize
\end{table}

The unconstrained classifier achieves a p-\% rule of 71.4\%.  By applying this thresholding strategy, 
we were able to obtain the same accuracy 
as the method of \cite{zafar2017fairness}, 
but with a higher $p$-\% rule of 78.3\% compared to 77.6\%.  Note that on this data, the method of \cite{zafar2017fairness} maxes out at a $p$-\% rule of 77.6\%.  That is, the method is limited in what $p$-\% rules may be achieved.  By contrast, the thresholding rule can achieve any desired parity level.  Subject to a $<1\%$ drop in accuracy relative to the DLP we can achieve a $p$-\% rule of $\sim{}100$\%.



\subsection{Examples on public datasets}


Finally, for reproducibility,
we repeat our experiments from Section \ref{sec:admissions}
on a variety of public datasets (code and data will be released at publication time).
Again we compare applying our simple thresholding scheme against the fairness constraint of \cite{zafar2017fairness}, considering a binary outcome 
and a single protected feature. 
Basic info about these datasets 
(including the prediction target and protected feature) 
is shown in Table~\ref{tab:datastats}.

The protocol we follow is the same as in Section \ref{sec:admissions}. 
Each of these datasets exhibits a certain degree of bias 
w.r.t.~the protected characteristic (Table~\ref{tab:results}), 
so no synthetic discrimination is applied. 
In Table~\ref{tab:results},
we compare (1) The $p$-\% rule obtained 
using the classifier of \cite{zafar2017fairness} 
compared to that of a na\"ive classifier 
(column k vs.~column h); and 
(2) The $p$-\% rule obtained 
when applying our thresholding strategy 
from Section~\ref{sec:threshold}.
As before, half of the data are withheld for testing.

First, we note that in most cases, 
the method of \cite{zafar2017fairness} increases 
the $p$-\% rule (column k vs.~h),
while maintaining an accuracy similar 
to that of unconstrained classification (column i vs.~f). 
One exception is the UCI-Credit dataset, 
in which \emph{both} the accuracy 
and the $p$-\% rule simultaneously decrease; 
although this is against our expectations, 
note that the optimization technique of \cite{zafar2017fairness} is an approximation scheme 
and does not offer accuracy guarantees in practice (nor can it in general achieve a $p$-\% rule of 100\%). 
However these details are implementation-specific 
and not the focus of this paper.
Second, as in Section~\ref{sec:threshold}, 
we note that the optimal thresholding strategy 
is able to offer a strictly larger $p$-\% rule 
(column l vs. k) at a given accuracy 
(in this case, the accuracy from column i). 
In most cases, we can obtain a $p$-\% rule 
of (close to) 100\% at the given accuracy.

\begin{table*}[t!]
\centering
\caption{Comparison between unconstrained classification, DLPs, and thresholding schemes.
Note that the $p$-\% rules from \cite{zafar2017fairness} were the strongest that could be obtained with their method; 
on complex datasets $p$-\% rules of 100\% 
are rarely obtained in practice, 
due to their specific approximation scheme.
Employee and Customer datasets are from IBM, the others are UCI datasets.
\label{tab:results}}
\scriptsize
\setlength{\tabcolsep}{3pt}
\begin{tabular}{lcccc|ccr|ccr|r}
\toprule
\multicolumn{5}{c|}{\textbf{basic statistics}} & \multicolumn{3}{c|}{\parbox{0.17\textwidth}{\centering \textbf{na\"ive (unconstrained)\\ classification}}} & \multicolumn{3}{c|}{\parbox{0.17\textwidth}{\centering \textbf{fair (constrained)\\ classification \citep{zafar2017fairness}}}} & \multicolumn{1}{c}{\parbox{0.08\textwidth}{\centering \textbf{optimal\\ threshold}}} \\
\midrule
\multicolumn{1}{c}{\textbf{dataset}} & \textbf{\%prot.} & 
\parbox{0.06\textwidth}{\centering \textbf{\%prot.\\ in +'ve}} & \parbox{0.09\textwidth}{\centering \textbf{\%non-prot.\\ in +'ve}} & \textbf{label $p$-\%} & \textbf{acc.} & \parbox{0.11\textwidth}{\centering \textbf{prot./non-prot.\\ in positive}} & \textbf{$p$-\%} & \textbf{acc.} & \parbox{0.11\textwidth}{\centering \textbf{prot./non-prot.\\ in positive}} & \textbf{$p$-\%} & \parbox{0.11\textwidth}{\centering \textbf{$p$-\% at\\ const.~acc.}}\\[1mm]
\multicolumn{1}{c}{a} & b & c & d & e & f & g & \multicolumn{1}{c|}{h} & i & j & \multicolumn{1}{c|}{k} & \multicolumn{1}{c}{l}\\
\midrule
Income    & 66.9\% & 30.6\% & 10.9\% & 35.8\% & 0.85 & \ 8\% / 25\% & 31\% & 0.85 & \ 7\% / 24\% & 29\% & 52.9\% \\
Marketing & 60.2\% & 14.1\% & 10.1\% & 71.9\% & 0.89 & \ 3\% / \ 4\% & 82\% & 0.89 & \ 3\% / \ 3\% & 102\% & 100.3\% \\
Credit    & 60.4\% & 24.1\% & 20.8\% & 86.0\% & 0.82 & 10\% / 12\% & 88\% & 0.74 & 21\% / 25\% & 85\% & 100.0\% \\
Employee  & 45.8\% & 19.2\% & 12.5\% & 65.0\% & 0.87 & \ 8\% / 12\% & 65\% & 0.86 & \ 8\% / 11\% & 69\% & 100.4\%\\
Customer  & 48.3\% & 33.0\% & 19.7\% & 59.7\% & 0.80 & 15\% / 30\% & 49\% & 0.79 & 16\% / 19\% & 84\% & 100.2\%\\
\bottomrule
\end{tabular}
\normalsize
\setlength{\tabcolsep}{6pt}
\end{table*}

We emphasize that the goal of our experiments 
is not to `beat' the method of \citep{zafar2017fairness}, 
or even to comment on any specific discrimination-aware classification scheme. 
Rather, we emphasize that \emph{any} DLP
is fundamentally upper-bounded 
(in terms of the $p$-\% rule/accuracy trade-off) 
by simple schemes that explicitly consider the protected feature. 
Our experiments validate this claim, 
and reveal that 
the two schemes make strikingly different decisions. 
While concealing the protected feature 
from the classifier may be conceptually desirable, 
practitioners should be aware of the consequences. 


\section{Discussion}
\label{section:discussion}

\vspace{-2mm}

\paragraph{Coming to terms with treatment disparity.}
Legal considerations aside, 
treatment disparity approaches 
have three advantages over DLPs: 
they optimally trade accuracy for representativeness, preserve rankings among members of each group,
and do no harm to members of the disadvantaged group. 
In addition,
treatment disparity has another advantage:
by setting class-dependent thresholds,
it's easier to understand
how treatment disparity impacts individuals. 
It seems plausible that policy-makers
could reason about thresholds 
to decide on the right trade-off
between group equality and individual fairness. 
By contrast the tuning parameters of DLPs
may be harder to reason about from a policy standpoint.
%
Several key challenges remain.  
Our theoretical arguments demonstrate 
that thresholding approaches are optimal 
in the setting where we assume complete knowledge 
of the data-generating distribution.  
It is not always clear how best 
to realize these gains in practice, 
where imbalanced or unrepresentative datasets 
can pose a significant obstacle to accurate estimation.

\vspace{-2mm}

\paragraph{Separating estimation from decision-making.}
In the context of algorithmic, 
or algorithm-supported decision-making, 
it's often useful to obtain not just a classification, 
but also an accurate probability estimate.  
These estimates could then be incorporated 
into the decision-theoretic part of the pipeline 
where appropriate measures could be taken to
align decisions with social values.   
By intervening at the modeling phase, 
DLPs distort the predicted probabilities themselves.  
It's not clear what the outputs of 
the resulting classifiers actually signify.
In unconstrained learning approaches, 
even if the label itself 
may reflect historical prejudice, 
one at least knows 
what is being estimated. 
This leaves open the possibility of intervening at decision time to promote more equal outcomes.  

\vspace{-2mm}

\paragraph{Fairness beyond disparate impact}
How best to quantify discrimination and unfairness 
remains an important open question.  
The CV scores and $p-\%$ rules 
offer one set of definitions,
but there are many other parity criteria
to which our results do not directly apply,
e.g., equality of opportunity \cite{hardt2016equality}.
Other  notions of fairness and the trade-offs between them 
have been studied \citep{joseph2016rawlsian,kleinberg2016inherent,chouldechova2017fairlong,berk2017fairness,ritov2017conditional}.  
In a recent paper, \citet{zafar2017parity} 
depart from parity-based definitions 
and propose instead a preference-based notion of fairness. \citet{dwork2017decoupled} address the problem 
of how best to incorporate information 
about protected characteristics 
for several of these other fairness criteria.

Problematically, research into fairness in ML 
is often motivated by the case 
in which our ground-truth data is tainted,
capturing existing discriminatory patterns.  
Characterizing different forms of data bias, how to detect them, and how to draw valid inference from such data
remain important outstanding challenges.

Even in settings where treatment disparity in favor of disadvantaged groups is an acceptable solution, 
questions remain of ``how'', ``how much?'' and ``when?''.
While in some cases treatment disparity may arguably be correcting for omitted variable bias historical discrimination, in other settings it may be viewed as itself a form of discrimination.
For example, in the United States, 
Asian students are simultaneously 
over-represented and discriminated 
against in higher education \citep{AsianNYT}.
Such policy judgments require a keen understanding 
and awareness of the social and historical context 
in which the algorithms are developed and meant to operate. 
Recent work on identifying proxy discrimination \citep{datta2017proxy} and causal formulations of fairness \citep{nabi2017fair,kilbertus2017avoiding,kusner2017counterfactual} 
offer some promising approaches to translating such understanding into technological solutions.  

\bibliographystyle{unsrtnat}
\bibliography{rethinking-disparate}

\end{document}